\documentclass{article}

\usepackage[preprint]{neurips_2022}

\usepackage[utf8]{inputenc} 
\usepackage[T1]{fontenc}    
\usepackage{hyperref}       
\usepackage{url}            
\usepackage{booktabs}       
\usepackage{amsfonts}       
\usepackage{nicefrac}       
\usepackage{microtype}      
\usepackage{xcolor}         

\usepackage{amsmath}
\usepackage{amssymb}
\usepackage{subcaption}
\usepackage{mathtools}
\usepackage{amsthm}
\usepackage[ruled,vlined]{algorithm2e}
\usepackage{dsfont}

\usepackage[capitalize,noabbrev]{cleveref}

\newcommand{\signalset}{\mathcal{X}}

\newcommand{\ntransf}{G}

\newcommand{\R}[1]{\mathbb{R}^{#1}}

\newcommand{\vect}[1]{\text{vec}(#1)}
\newcommand{\rangeA}{\mathcal{R}_A}
\newcommand{\rangeAg}{\mathcal{R}_{A_g}}
\newcommand{\nullA}{\mathcal{N}_A}
\newcommand{\rk}[1]{\text{rank}(#1)}
\newcommand{\nullAg}{\mathcal{N}_{A_g}}

\newcommand{\rev}[1]{\textcolor{black}{#1}}
\newcommand{\group}{\mathcal{G}}

\DeclareMathOperator*{\argmin}{arg\,min}

\theoremstyle{plain}
\newtheorem{theorem}{Theorem}[section]
\newtheorem{proposition}[theorem]{Proposition}
\newtheorem{lemma}[theorem]{Lemma}

\theoremstyle{definition}

\theoremstyle{remark}

\newcommand{\JT}[1]{{\color{black}{#1}}}
\crefformat{equation}{(#2#1#3)}

\title{Unsupervised Learning From Incomplete Measurements for Inverse Problems}

\author{%
  Juli\'an Tachella \\
  Laboratoire de Physique\\
  CNRS \& ENSL\\
  Lyon, France \\
  \texttt{julian.tachella@cnrs.fr} \\
   \And
   Dongdong Chen \\
   School of Engineering \\
   University of Edinburgh \\
   Edinburgh, UK \\
   \texttt{d.chen@ed.ac.uk} \\
   \And
   Mike Davies \\
   School of Engineering \\
   University of Edinburgh \\
   Edinburgh, UK \\
   \texttt{mike.davies@ed.ac.uk} \\
}

\begin{document}

\maketitle

\begin{abstract}
In many real-world inverse problems, only incomplete measurement data are available for training which can pose a problem for learning a reconstruction function. Indeed, unsupervised learning using a fixed incomplete measurement process is impossible in general, as there is no information in the nullspace of the measurement operator. This limitation can be overcome by using measurements from multiple operators. While this idea has been successfully applied in various applications, a precise characterization of the conditions for learning is still lacking. In this paper, we fill this gap by presenting necessary and sufficient conditions for learning the underlying signal model needed for reconstruction which indicate the interplay between the number of distinct measurement operators, the number of measurements per operator, the dimension of the model and the dimension of the signals. Furthermore, we propose a novel and conceptually simple unsupervised learning loss which only requires access to incomplete measurement data and achieves a performance on par with supervised learning when the sufficient condition is verified. We validate our theoretical bounds and demonstrate the advantages of the proposed unsupervised loss compared to previous methods via a series of experiments on various imaging inverse problems, such as accelerated magnetic resonance imaging, compressed sensing and image inpainting.
\end{abstract}

\section{Introduction}
\label{sec:intro}
In multiple sensing applications, we observe measurements $y\in\mathbb{R}^{m}$ associated with a signal $x\in\signalset\subset \mathbb{R}^{n}$, through the forward process
\begin{equation}
    y = Ax + \epsilon
\end{equation}
where $A\in \mathbb{R}^{n\times n}$ is a linear measurement operator and $\epsilon$ denotes the noise affecting the measurements.
This is the case of computed tomography~\cite{jin2017deep}, depth ranging~\cite{rapp2020advances} and non-line-of-sight imaging~\cite{o2018confocal} to name a few. Estimating $x$ from $y$ is generally an ill-posed inverse problem due to the incomplete operator 
$A$ with $m< n$ and the presence of noise. Knowledge of the signal model is required to make this problem well-posed.



In many cases, obtaining ground-truth reconstructions $x$ to learn the reconstruction function $y\mapsto x$ might be very expensive or even impossible. For example, in medical imaging, it is not always possible to obtain fully sampled images of patients as they require long acquisition times. In astronomical imaging, it is impossible to obtain ground-truth references due to physical limitations. In electron-microscopy imaging~\cite{gupta2020multi}, we can only measure 2D projections of a molecule.
In these settings, we can only access measurements $y$ for learning.
Moreover, if the measurement process $A$ is incomplete, it is fundamentally impossible to learn the model with only measurements $y$, as there is no information about the model in the nullspace of $A$. Thus, we end up with a chicken-and-egg problem: in order to reconstruct $x$ we need the reconstruction function, but to learn this function we require some reconstructed samples $x$. 

This fundamental limitation can be overcome by  using information from multiple incomplete sensing operators $A_1,\dots,A_{\ntransf}$, the general principle being that each operator can provide additional information about the signal model if it has a different nullspace. 
For example, in the image inpainting problem, Studer and Baraniuk~\cite{studer2012dictionary} used the fact that  the set of missing pixels may vary between observed images to learn a sparse dictionary model and reconstruct the images. Yang et al.~\cite{yang2015mixture} used multiple operators to learn a Gaussian mixture model in the context of hyperspectral imaging and high-speed video. Bora et al.~\cite{bora2018ambientgan} exploited this idea for learning a generative model in various imaging problems such as deblurring and compressed sensing. Matrix completion methods~\cite{candes2009exact} exploit a similar principle, as the missing entries of each column (i.e., signal) are generally different. Ideally we would like to learn the reconstruction function and signal model from only a small number of different measurement operators. We are thus motivated to determine typically how many such operators are required.

The problem can be formalized as follows. 
We first focus on the noiseless case to study the intrinsic identifiability problems associated to having only incomplete measurement data. The effect of noise will be discussed in~\Cref{sec: low-dim models}.
We assume that we observe a set of $N$ training samples $y_i$, where the $i$th signal is observed via $A_{g_i}\in \mathbb{R}^{m\times n}$, one of $\ntransf$ linear operators
, i.e.,
\begin{equation}
    y_i = A_{g_i} x_i
\end{equation}
where $g_i \in \{1,\dots,\ntransf\}$ and $i=1,\dots,N$.  While we assume that the measurement operator $A_{g_i}$ is known for all observed signals, it is important to note that we do not know a priori if two observations $(y_i,A_{g_i})$ and $(y_{i'},A_{g_{i'}})$ are related to the same signal \JT{as in Noise2Noise~\cite{lehtinen2018noise2noise}.} 
There are two natural questions regarding this learning problem:

\begin{description}
    \item[Q1. Signal Recovery] Is there a unique signal $x\in\signalset$ which verifies the measurements $y=A_gx$? In other words, is  the reconstruction function $f:(y,A_g)\mapsto x$ one-to-one?
    \item[Q2. Model Identification] Can we uniquely identify the signal model from measurement data alone obtained via \JT{the incomplete operators $A_1,\dots,A_{\ntransf}$}?
\end{description}
In general, there can be a unique solution for neither problem, just one or both. There might be a unique solution for signal recovery if the model is known, but it might be impossible to learn the model in the first place~(e.g., blind compressed sensing \cite{gleichman2011blind}). The converse is also possible, that is, uniquely identifying a model without having enough measurements per sample to uniquely identify the associated signal (e.g., subspace learning from rank-1 projections~\cite{chen2015rankone}).

The answer to \textbf{Q1} is well-known from generalized compressed sensing theory, see for example~\cite{bourrier2014fundamental}. Unique signal recovery is possible if the signal model is low-dimensional, i.e., if $A_g$ has  $m>2k$ measurements, where $k$ is the model dimension. On the other hand, \textbf{Q2} has been mostly studied in the context of matrix completion, where the set of signals is assumed to lie in a low-dimensional subspace of $\mathbb{R}^{n}$. Bora et al.~\cite{bora2018ambientgan} presented some results in the general setting, but only for the case where $G=\infty$ which is quite restrictive. In this paper, we provide sharp necessary and sufficient conditions which hold for any low-dimensional distribution (beyond linear subspaces) and only require a finite number of operators $G$.

If the conditions for signal recovery  and model identification are satisfied, we can expect to learn the reconstruction function from measurement data alone. We introduce a new unsupervised learning objective which can be used to learn the reconstruction function $f:(y,A_g)\mapsto x$, and provides performances on par with supervised learning when the sufficient conditions are met. The main contributions of this paper are as follows:

\begin{itemize}
\item We show that unsupervised learning from a finite number of incomplete measurement operators is only possible if the model is low-dimensional.  More precisely, we show that $m \geq n/ \ntransf$ measurements per operator are necessary for learning, and that for almost every set of $\ntransf$ operators, $m > k + n/G$ measurements per operator are sufficient.
\item We propose a new unsupervised loss for learning the reconstruction function that only requires incomplete measurement data, which empirically obtains a performance on par with fully supervised methods when the sufficient condition $m > k + n/G$ is met.
\item A series of experiments demonstrate that our bounds accurately characterize the performance of unsupervised approaches on synthetic and real datasets, and that the proposed unsupervised approach outperforms previous methods in various inverse problems.
\end{itemize}

\subsection{Related Work}

\paragraph{Blind Compressed Sensing}
The fundamental limitation of failing to learn a signal model from incomplete (compressed) measurements data goes back to  blind compressed sensing~\cite{gleichman2011blind} for the specific case of models exploiting sparsity on an orthogonal dictionary. In order to learn the dictionary from incomplete data, \cite{gleichman2011blind} imposed additional constraints on the dictionary, while some subsequent papers~\cite{silva2011blind,aghagolzadeh2015new} removed these assumptions by proposing to use multiple operators $A_g$ as studied here. This paper can be seen as a generalization of such results to more general signal models. 

\paragraph{Matrix Completion}
Matrix completion consists of inferring missing entries of a data matrix $Y = [y_1,\dots,y_N]$, whose columns are generally inpainted samples from a low-dimensional distribution, i.e., $y_i = A_{g_i}x_i$ where the operators $A_{g_i}$ randomly select a subset of $m$ entries of the signal $x_i$. This problem can be viewed as the combination of  model identification, i.e., identifying the low-rank subspace that the columns of $X=[x_1,\dots,x_N]$ belong to, and signal recovery, i.e., reconstructing the individual columns.  Assuming that the samples belong to a $k$-dimensional subspace can be imposed by recovering a rank-$k$ signal matrix $X$ from $Y$. If the columns are sampled via $\ntransf$ sufficiently different patterns $A_{g_i}$ with the same number of entries $m$, a sufficient condition~\cite{pimentel2016characterization} for uniquely recovering almost every subspace model is\footnote{A larger number of measurements $m =\mathcal{O}(k\log n)$ is required to guarantee a stable recovery when the number of patterns $\ntransf$ is large~\cite{candes2009exact}.} $m \geq  (1-1/\ntransf)k + n/\ntransf$. 

A similar condition was  shown in~\cite{pimentel2016information} for the case of \emph{high-rank} matrix completion~\cite{eriksson2012high}, which arises when the samples $x_i$ belong to a union of $k$-dimensional subspaces.
We show that model identification is possible for almost every set of $\ntransf$ operators with $m > k + n/\ntransf$ measurements, however the  theory presented here goes beyond linear subspaces, being also valid for general low-dimensional models. 

\paragraph{Deep Nets for Inverse Problems}
Despite providing very competitive results, most deep learning based solvers require measurements and signal pairs $(x_i,y_i)$ (or at least clean signals $x_i$) in order to learn the reconstruction function $y\mapsto x$ from incomplete measurements. A first step to overcome this limitation is due to Noise2Noise~\cite{lehtinen2018noise2noise}, where the authors show that it is possible to learn from only noisy samples. However, their ideas only apply to denoising settings where there is a trivial nullspace, as the operator $A$ is the identity matrix.  \rev{This approach was extended in~\cite{xia2019training} to the case where two measurements are observed per signal, each associated with a different operator.}
Yaman et al.~\cite{yaman2020mri} and Artifact2Artifact~\cite{liu2020rare} empirically showed that it is possible to exploit different measurement operators to learn the reconstruction function in the context of magnetic resonance imaging (MRI). AmbientGAN~\cite{bora2018ambientgan} proposed to learn a signal distribution from only incomplete measurements using multiple forward operators, however they only provide reconstruction guarantees for the case where an infinite number of operators $A_g$ is available\footnote{Their result relies on the Cram\'er-Wold theorem, which is discussed in~\Cref{sec: highdim}.}, a condition that is not met in practice.  

\rev{Another line of work focuses on learning using measurements from a single incomplete operator. The works in~\cite{metzler2018unsupervised,zhussip2019training} use the large system limit properties of random compressed sensing operators to learn from measurements alone. The equivariant imaging approach~\cite{chen2021equivariant,chen2022robust} leverages invariance of the signal set to a group of transformations to learn from general incomplete operators.}

\section{Signal Recovery Preliminaries} \label{subsec:CS preliminary}

We denote the nullspace of $A$ as $\nullA$. Its complement, the range space of the pseudo-inverse $A^{\dagger}$, is denoted as $\rangeA$, where $\rangeA \oplus \nullA = \R{n}$ and $\oplus$ denotes the direct sum. Throughout the paper, we assume that the signals are sampled from a measure $\mu$ supported on the signal set $\signalset\subset \R{n}$. Signal recovery has a unique solution if and only if the forward operator $x\mapsto y$ is one-to-one, i.e., if for every pair of signals $x_1,x_2\in\signalset$ where $x_1\neq x_2$ we have that 
\begin{align}
    Ax_1 \neq Ax_2 \\
    A(x_1-x_2) \neq 0 
\end{align}
In other words, there is no vector $x_1-x_2\neq 0$ in the nullspace of $A$.
It is well-known that this is only possible if the signal set $\signalset$ is low-dimensional. There are multiple ways to define the notion of dimensionality of a set in $\R{n}$. In this paper, we focus on the upper box-counting dimension which is defined for a compact subset $S\subset\R{n}$ as
\begin{equation}
   \text{boxdim}(S) = \lim \sup_{\epsilon\to0}  \frac{\log N(S,\epsilon)}{-\log \epsilon}
\end{equation}
where $N(S,\epsilon)$ is the minimum number of closed balls of radius $\epsilon$ with respect to the norm $\|\cdot\|$ that are required to cover $S$. This definition of dimension covers both well-behaved models such as compact manifolds and more general  low-dimensional sets.  The mapping $x\mapsto y$ is one-to-one for almost every forward operator $A\in\R{m\times n}$ if~\cite{sauer1991embedology}
\begin{equation}\label{eq: one-to-one}
    m>\text{boxdim}(\Delta\signalset)
\end{equation}
where $\Delta\signalset$ denotes the normalized secant set which is defined as
\begin{equation}
   \Delta\signalset = \{ \Delta x \in \R{n} | \; \Delta x = \frac{x_2 -x_1}{\| x_2 -x_1\|},  x_1,x_2\in\signalset, 
   \; x_2\neq x_1 \}.
\end{equation}
The term \emph{almost every} means that the complement has Lebesgue measure 0 in the space of linear measurement operators $\R{m\times n}$.  The normalized secant set  of models of dimension $k$  generally has dimension $2k$, requiring  $m>2k$ measurements to ensure signal recovery.
For example, the  union of $k$-dimensional subspaces requires at least $2k$ measurements\footnote{While the bound in~\Cref{eq: one-to-one} guarantees \emph{unique} signal recovery, more measurements  (e.g., an additional factor of $\mathcal{O}(\log n)$ measurements) are typically necessary in order to have a \emph{stable} inverse $f:y\mapsto x$, i.e., possessing a certain Lipschitz constant. A detailed discussion can be found for example in~\cite{bourrier2014fundamental}.}  to guarantee one-to-oneness~\cite{blumensath2009uos}. This includes well-known models such as $k$-sparse models (e.g., convolutional sparse coding~\cite{bristow2013fast}) and co-sparse models (e.g., total variation~\cite{rudin1992nonlinear}).
In the regime $k<m\leq 2k$, the subset of signals where one-to-oneness fails is at most $(2k-m)$-dimensional~\cite{sauer1991embedology}. 

\section{Uniqueness of Any Model?} \label{sec: highdim}
A natural first question when considering uniqueness of the model is: can we recover any \JT{probability} measure $\mu$ observed via forward operators $A_1,\dots,A_G$, even in the case where \JT{its support} $\signalset$ is the full $\mathbb{R}^{n}$? We show that, in general, the answer is no.

Uniqueness can be analysed from the point of view of the characteristic function of $\mu$, defined as $\varphi(w) = \mathbb{E}\{e^{\mathrm{i}w^{\top}x}\}$ where the expectation is taken with respect to $\mu$ and $\mathrm{i}=\sqrt{-1}$ is the imaginary unit. If two distributions have the same characteristic function, then they are necessarily the same almost everywhere. Each forward operator provides information about a subspace of the characteristic function  as
\begin{align}
\label{eq: stat}
\mathbb{E}\{e^{\mathrm{i}w^{\top}A_g^{\dagger}y}\} 
&= \mathbb{E}\{e^{\mathrm{i}w^{\top}A_g^{\dagger}A_gx}\} \\
&= \mathbb{E}\{e^{\mathrm{i}(A_g^{\dagger}A_g w)^{\top}x}\} \\
&= \varphi(A_g^{\dagger}A_gw)
\end{align}
where $A_g^{\dagger}A_g$ is a linear projection onto the subspace $\rangeAg$.
Given that $m<n$, the characteristic function is only observed in the subspaces $\rangeAg$ for all $g\in \{1,\dots,\ntransf\}$. For any finite number of operators, the union of these subspaces does not cover the whole $\mathbb{R}^{n}$, and hence there is loss of information, i.e., the signal model cannot be uniquely identified. 

In the case of an infinite number of operators $G=\infty$,  
the Cram\'er-Wold theorem guarantees uniqueness of the signal distribution if all possible one dimensional projections ($m=1$) are available~\cite{cramer1936some,bora2018ambientgan}. 
However, in most practical settings we can only access a finite number of  operators and many distributions will be non-identifiable.

\section{Uniqueness of Low-Dimensional Models} \label{sec: low-dim models}
Most models appearing in signal processing and machine learning are assumed to be approximately low-dimensional, with a dimension $k$ which is much lower than the ambient dimension $n$. As discussed in~\Cref{subsec:CS preliminary}, the low-dimensional property is the key to obtain stable reconstructions, e.g., in compressed sensing. 
In the rest of the paper, we impose the following assumptions on the model:
\begin{enumerate}
    \item[\textbf{A1}] The signal set $\signalset$ is either 
    \begin{enumerate}
        \item  A bounded set with box-counting dimension $k$.
        \item An unbounded conic set whose  intersection with the unit sphere has box-counting dimension $k-1$.
    \end{enumerate}
\end{enumerate}

This assumption has been widely adopted in the inverse problems literature, as it is a necessary assumption to guarantee signal recovery. Our definition of dimension covers most 
models used in practice, such as simple subspace models, union of subspaces (convolutional sparse coding models, $k$-sparse models), low-rank matrices and compact manifolds. It is worth noting that dimension is a property of the dataset and thus independent of the specific algorithm used for learning.


In the rest of the paper, we focus on conditions for the identification of the support $\signalset$ instead of the signal distribution $\mu$, due to the following observation: if there is a one-to-one reconstruction function (which happens for almost every $A$ with $m>2k$ as explained in~\cref{subsec:CS preliminary}), uniqueness of the support implies uniqueness of $\mu$.
 If $\signalset$ is known and there is a measurable one-to-one mapping from each observed measurement $y$ to $\signalset$, then it is possible to obtain $\mu$ as the push-forward of the measurement distribution.

Before delving into the main theorem, we present a simple example which provides intuition of how a low-dimensional model can be learned via multiple projections $A_g$:

\paragraph{Learning a one-dimensional subspace} \label{ex: 1d line}
Consider a toy signal model with support $\signalset\subset\mathbb{R}^3$ which consists of a one-dimensional linear subspace spanned by $\phi = [1,1,1]^{\top}$, and $G=3$ measurement operators $A_1,A_2,A_3\in \mathbb{R}^{2\times 3}$ which project the signals into the $x(3)=0$, $x(2)=0$ and $x(1) = 0$ planes respectively, where $x(i)$ denotes the $i$th entry of the vector $x$. 
The example is illustrated in \Cref{fig:toy_line}. The first operator $A_1$ imposes a constraint on $\signalset$, that is, every $x\in\signalset$ should verify $x(1)-x(2)=0$. Without more operators providing additional information about $\signalset$, this constraint yields a plane containing $\signalset$, and there are infinitely many one-dimensional models that would fit the training data perfectly. However, the additional operator $A_2$ adds the constraint $x(2)-x(3)=0$, which is sufficient to uniquely identify $\signalset$ as 
\begin{align*}
    \hat{\signalset} = \signalset =\{v \in \mathbb{R}^3 | \;  v(1) - v(2) = v(2) - v(3) = 0 \}
\end{align*}
is the desired 1-dimensional subspace.
Finally, note that in this case the operator $A_3$ does not restrict the signal set further, as the constraint $x(1)-x(3)=0$ is verified by the other two constraints.

\begin{figure*}[h]
\centering
\includegraphics[width=1\textwidth]{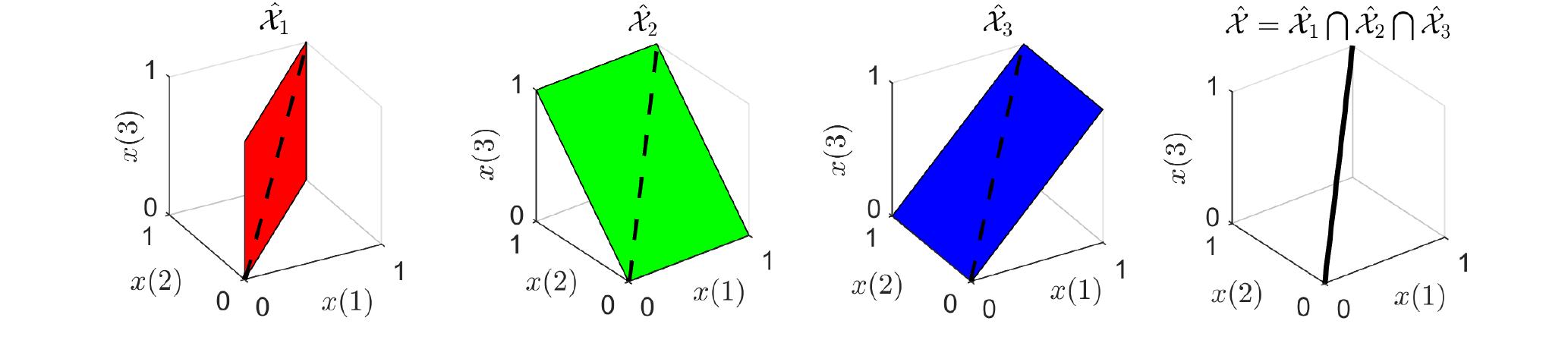}
\caption{Toy example of a 1-dimensional subspace embedded in $\R{3}$. If we only observe the projection of the signal set into the plane $x(3)=0$, then there are infinite possible lines that are consistent with the measurements (red plane). Adding the projection into the $x(1)=0$ plane, allows us to uniquely identify the signal model.}
\label{fig:toy_line}
\end{figure*}

The ideas from the one-dimensional subspace example can be generalized and formalized as follows: for each projection $A_g$, we can constrain the model support $\signalset$ by considering the  set \begin{equation} \label{eq: constraint Ag}
    \hat{\signalset}_g = \{v \in \mathbb{R}^{n} | \; v = \hat{x}_g + u, \; \hat{x}_g\in \signalset, \; u\in \nullAg \}
\end{equation}
which has dimension at most $n-(m-k)$. Note that the true signal model is a subset of $\hat{\signalset}_g$. The inferred signal set belongs to the intersection of these sets 
\begin{equation}
    \hat{\signalset}=\bigcap_{g\in\group} \hat{\signalset}_g 
\end{equation}
which can be expressed concisely as
\begin{equation}
    \label{eq:inferred set}
  \hat{\signalset} =  \{ v\in \mathbb{R}^{n} | \; A_g(x_{g} -v) = 0, \; g=1,\dots,\ntransf, x_1,\dots,x_{\ntransf}\in \signalset \}
\end{equation} 
Even though we have derived the set $\hat{\signalset}$ from a purely geometrical argument, the constraints in \Cref{eq:inferred set} also offer a simple algebraic intuition: the inferred signal set consists of the points $v\in \R{n}$ which verify
the following system of equations 
\begin{equation} \label{eq: algo_x}
    \begin{bmatrix}
     A_1 \\ 
     \vdots \\
     A_{\ntransf} 
    \end{bmatrix} v =   \begin{bmatrix}
     A_1x_1 \\ 
     \vdots \\
     A_{\ntransf} x_{\ntransf}
    \end{bmatrix}.
\end{equation}
for all possible choices of $\ntransf$ points $x_1,\dots,x_{\ntransf}$ in $\signalset$.
In other words, given a dataset of $N$ incomplete measurements $\{A_{g_i}x_i\}_{i=1}^{N}$, it is possible to build $\hat{\signalset}$ by trying all the possible combinations of $\ntransf$ samples\footnote{Despite providing a good intuition, this procedure for estimating $\signalset$ is far from being practical as it would require an infinite number of observed samples if the dimension of the signal set is not trivial $k>0$.} and keeping only the points $v$ which are the solutions of~\Cref{eq: algo_x}.

It is trivial to see that $\signalset\subseteq \hat{\signalset}$, but when can we guarantee $\signalset=\hat{\signalset}$? 
As in the previous toy example, if there are not enough constraints, e.g., if we have a single $A$ and no additional measurement operators, the inferred set will have a dimension larger than $k$, containing undesired aliases. In particular, we have the following lower bound on the minimum number of measurements:

\begin{proposition}[Theorem 1 in~\cite{chen2021equivariant}] \label{prop: necessary multA}
\rev{A necessary condition for model uniqueness \JT{from the measurement sets $\{A_g \signalset\}_{g=1}^{\ntransf}$} is that \begin{equation}\label{eq: rank condition}
    \rk{\begin{bmatrix}
    A_{1} \\
    \vdots \\
    A_{\ntransf}
    \end{bmatrix}} = n
\end{equation} and thus $m\geq n/\ntransf$.}
\end{proposition}
\begin{proof}
In order to have model uniqueness, the system in~\cref{eq: algo_x} should only admit a solution  if $v=x_1=\dots=x_\ntransf$. \rev{If the rank condition in~\cref{eq: rank condition} is not satisfied, there is more than one solution for any choice of $x_1,\dots,x_\ntransf\in\signalset$.}  
\end{proof}

Note that this necessary condition does not take into account the dimension of the model. As discussed in~\Cref{sec: highdim}, a sufficient condition for model uniqueness must depend on the dimension of the signal set $k$. Our main theorem shows that $k$ additional measurements per operator are sufficient for model identification:


\begin{theorem} \label{theo: multiple op}
For almost every set of $\ntransf$ mappings $A_1,\dots,A_{\ntransf}\in \mathbb{R}^{m\times n}$, \rev{under assumption \textup{\textbf{A1}}} the signal model $\signalset$ can be uniquely recovered \JT{from the measurement sets $\{A_g \signalset\}_{g=1}^{\ntransf}$} if the number of measurements \JT{per operator} verifies $m> k + n/\ntransf$.
\end{theorem}

The proof is included in~\Cref{sec:proofs}. If we have a large number of independent operators $\ntransf\geq n$, \Cref{theo: multiple op} states that only $m>k+1$ measurements are sufficient for model identification, which is slightly smaller (if the model is not trivial, i.e., $k>1$) than the number of measurements typically needed for signal recovery $m>2k$. In this case, it is possible to uniquely identify the model, without necessarily having a unique reconstruction of each observed signal. However, as discussed in~\Cref{subsec:CS preliminary}, for $k<m\leq 2k$, the subset of signals which cannot be uniquely recovered is at most $(2k-m)$-dimensional.




\paragraph{Operators with Different Number of Measurements} \label{subsec: different m}
The results of the previous subsections can be easily extended to the setting where each measurement operator has a different number of measurements, i.e. $A_1\in \R{m_1 \times n},\dots, A_{\ntransf}\in \R{m_{\ntransf} \times n}$. In this case, the necessary condition in \Cref{prop: necessary multA} is $\sum_{g=1}^{\ntransf} m_g \geq n$, and the sufficient condition in~\Cref{theo: multiple op} is $\frac{1}{\ntransf}\sum_{g=1}^{\ntransf} m_g > k + n/\ntransf$.
As the proofs mirror the ones of \Cref{prop: necessary multA} and  \Cref{theo: multiple op}, we leave the details to the reader.

\paragraph{Noisy measurement data}
Surprisingly, the results of this section are also theoretically valid if the measurements are corrupted by independent additive noise $\epsilon$, i.e., $y = A_{g}x + \epsilon$, as long as the noise distribution is \emph{known} and has a nowhere zero characteristic function (e.g., Gaussian noise):
\begin{proposition} \label{prop:noise}
For a fixed noise distribution, if its characteristic function is nowhere zero, then there is a one-to-one mapping between the space of clean measurement distributions and noise measurement distributions. 
\end{proposition}
The proof is included in~\Cref{sec:proofs}. If the clean measurement distribution can be uniquely identified from the noisy distribution, the results in \Cref{theo: multiple op} and \Cref{prop: necessary multA} also carry over to the noisy setting. Note that this only guarantees model identifiability and makes no claims on the sample complexity of any learning process.

\section{Algorithms}\label{sec: algos}
Unsupervised algorithms mainly come in two flavours: we can first learn a model $\hat{\signalset}$ to then reconstruct by projecting measurements into this set, or we can attempt to directly learn the reconstruction function parameterized by a deep network.
\subsection{Learn the Model and Reconstruct}
\paragraph{Dictionary and Subspace Learning}  If $\signalset$ is (approximately) a union of subspaces~\cite{gleichman2011blind} or a single subspace~\cite{candes2009exact}, we can learn a model by
\begin{equation}
    \argmin_{z,D} \mathbb{E}_{(y,g)}\| y- A_gDz \|^2 + \rho_1(D) + \rho_2(z)
\end{equation}
where $\rho_1(D)$ and $\rho_2(z)$ are regularisation terms that promote low-dimensional solutions, e.g., sparse codes $z$ if $\signalset$ is a union of subspaces. At test time, the dictionary is fixed and the optimization is performed over the codes only. 

\paragraph{AmbientGAN} Complex datasets are often better modelled by a generative network $f:\R{k}\mapsto\R{n}$ whose input is a low-dimensional latent code $z\in\R{k}$.  The generative model can be learned using an adversarial strategy, i.e.,
\begin{equation} \label{eq:conditionalGAN}
    \argmin_f \max_d \mathbb{E}_{(y,g)} q\{ d(A_{g}^{\dagger}y) \} +  \mathbb{E}_{z}\mathbb{E}_{g} q\left\{ 1- d\left(A_{g}^{\dagger}A_{g}f(z)\right) \right\}
\end{equation}
where $d:\R{n}\mapsto\R{n}$ is the discriminator network which compares measurements in the image domain, $z$ is usually sampled from a Gaussian distribution, and $q(t)=\log(t)$ for standard GANs and $q(t)=t$ for Wasserstein GANs.
At test time, the reconstruction can be obtained by finding the latent code that best fits the measurements
$\hat{x} = f(\argmin_z \|y - A_gf(z) \|^2)$, as in~\cite{bora2017compressed}.

\subsection{Learn to Reconstruct}
Another approach consists in learning directly the reconstruction function $f: \R{m}\times \R{m\times n}\mapsto \R{n}$ whose inputs are the measurement $y$ and the associated operator $A_g$, and the output is the reconstructed signal $x$. The reconstruction function can have either a denoiser form, $f(y,A_g) = \tilde{f}(A_g^{\dagger}y)$ where $\tilde{f}$ is  independent of $A_g$~\cite{jin2017deep}, or a more complex unrolled structure with many denoising and gradient steps~\cite{monga2021unrolled}.
\paragraph{Measurement Splitting}
Inspired by the Noise2Noise approaches~\cite{lehtinen2018noise2noise,xia2019training}, some self-supervised methods~\cite{yaman2020mri,liu2020rare} split each measurement into two parts, $y^{\top} = [y_{1}^{\top},y_{2}^{\top}]$, such that the input is $y_{1}$ and the target is $y_{2}$. These methods can be summarised as minimising the following loss 
\begin{equation} \label{eq:a2a}
    \argmin_f \mathbb{E}_{(y,g)}\| y_{2} - A_{g,2} f(y_{1}, A_{g,1}) \|^2 
\end{equation}
 where $y_{1}=A_{g,1}x+\epsilon_1$ and $y_{2}=A_{g,2}x+\epsilon_2$.
 This approach suffers from the fact that $f$ does not use all the available information in a given measurement, as it attempts to solve a harder reconstruction problem associated with $y_{1}=A_{g,1}x+\epsilon_1$. As reconstruction networks often fail to generalise to operators with more measurements~\cite{antun2020instabilities}, this method can suffer from suboptimal reconstructions at test time.

\paragraph{Proposed Method}
\rev{The analysis in~\cref{sec: low-dim models} shows that model identifiability necessarily requires that reconstructed signals are consistent with all operators $A_1,\dots,A_{\ntransf}$.} Thus, we propose an unsupervised loss that ensures consistency across all projections $A_g$, that is
\begin{equation} \label{eq:moi}
    \argmin_f \mathbb{E}_{(y,g)} \left\{\| y - A_{g} f(y, A_{g}) \|^2 +  \mathbb{E}_{s} \| \hat{x} - f(A_{s}\hat{x}, A_{s}) \|^2 \right\}
\end{equation}
where $\hat{x} = f(y, A_{g})$. The first term ensures measurement consistency $y = A_{g}f(y,A_{g})$, whereas the second term enforces consistency across operators, i.e., $ f(y,A_{g}) = f(A_s f(y,A_{g}),A_s)$ for all $g\neq s$. Crucially, the second term prevents the network from learning the trivial pseudo-inverse $f(y,A_g) = A_g^{\dagger}y$. In practice, we choose an operator $A_s$ uniformly at random per minibatch. Thus, compared the supervised case, we only require an additional evaluation of $f$ and $A_s$ per minibatch.  We coin this approach multi-operator imaging (MOI). 
This loss overcomes the main disadvantages of previous approaches: it doesn't require training a discriminator network, and it learns to reconstruct using all the available information in each measurement $y$. 


\section{Experiments} 

In this section, we present a series of experiments where the goal is to learn the reconstruction function with deep networks using real datasets. Experiments on low-dimensional subspace learning using synthetic datasets are presented in~\Cref{subsec: subspace}. All our experiments were performed using an internal cluster of 4 NVIDIA RTX 3090 GPUs with a total compute time of approximately 48 hs.

\paragraph{Compressed Sensing and Inpainting with MNIST}
We evaluate the theoretical bounds on the MNIST dataset, based on the well known approximation of its box-counting dimension $k\approx 12$~\cite{hein2005intrinsic}. The dataset contains $N=60000$ training samples, and these are partitioned such that $N/G$ different samples are observed via each operator. The forward operators are compressed sensing (CS) matrices with entries sampled from a Gaussian distribution with zero mean and variance $m^{-1}$. The test set consists of $10000$ samples, which are also randomly divided into $G$ parts, one per operator. 
To evaluate the theoretical bounds, we attempt to minimize the impact of the inductive bias of the networks' architecture~\cite{ulyanov2018deep,tachella2021nonlocal} by using a network with 5 fully connected layers and relu non-linearities.



\Cref{fig:results moi mnist cs} shows the average test peak-signal-to-noise ratio (PSNR) achieved by the model trained using the proposed MOI loss for $G=1,10,20,30,40$ and $m=50,100,200,300,400$. The results follow the bound presented in~\Cref{sec: low-dim models} (indicated by the red dashed line), as the network is  only able to  learn the reconstruction mapping when the sufficient condition $m>k+n/G$ is verified. In sensing regimes below this condition, the performance is similar to the pseudo-inverse $A_g^{\dagger}$. 

We also evaluate the reconstruction for a different number $G$ of random inpainting masks and different rates $m$. The inpainting operators have a diagonal structure which has zero measure in $\mathbb{R}^{m\times n}$, however our sufficient condition still provides a reasonable lower bound on predicting the performance, as shown in~\Cref{fig:results moi mnist inp}. It is likely that due to the coherence between measurement operators and images (both operators and MNIST images are sparse), more measurements are required to obtain good reconstructions than in the CS case. 

 \begin{figure}[h]
\centering
 \begin{subfigure}{0.35\textwidth}
     \centering
     \includegraphics[width=1\textwidth]{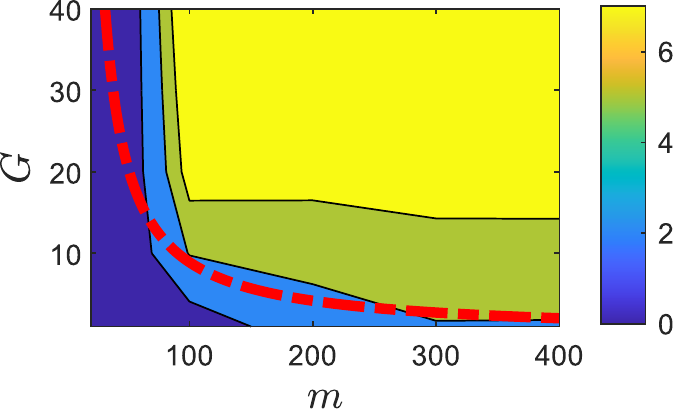}
     \caption{Compressed Sensing}
     \label{fig:results moi mnist cs}
 \end{subfigure}
 \begin{subfigure}{0.35\textwidth}
     \centering
\includegraphics[width=1\columnwidth]{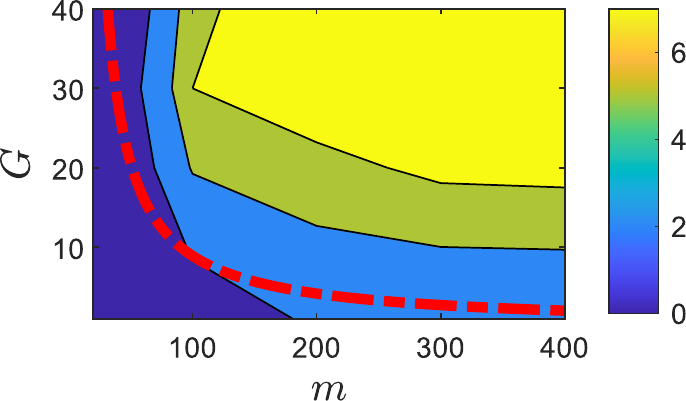}
     \caption{Inpainting}
     \label{fig:results moi mnist inp}
 \end{subfigure}
\caption{Average test PSNR improvement in dB over the pseudo-inverse for the MNIST dataset using the proposed training loss, for different number of CS operators  or inpainting masks and measurements per operator. The curve in red shows the necessary condition of \Cref{theo: multiple op}, $m>k+n/\ntransf$.}
\label{fig:mnist_results}
\end{figure}

\begin{table}[t]
\begin{center}
\fontsize{8}{12}\selectfont
\begin{tabular}{c|cccc}
 Inpainting/CelebA & $A^{\dagger}y$ & AmbientGAN & MOI (ours)  & Supervised\\ \hline
 & 9.05$\pm$1.65 & 29.57$\pm$1.24 & \textbf{34.05$\pm$3.77}&  36.21$\pm$3.76 \\
\end{tabular}

\begin{tabular}{c|cccc}
Acc. MRI/FastMRI   & $A^{\dagger}y$ & Meas. Splitting & MOI (ours)  & Supervised\\ \hline
\JT{Denoiser $f$} & 25.77$\pm$2.71 & 28.72$\pm$1.64 &  \textbf{29.51$\pm$1.85} & 31.45$\pm$1.98 \\
\JT{Unrolled $f$}  & 25.77$\pm$2.71 & 29.47$\pm$2.02 &  \textbf{31.39$\pm$2.17} & 32.42$\pm$2.44 \\
\end{tabular}

\end{center}
\caption{Comparison of supervised and unsupervised learning methods for inpainting and accelerated MRI. Reported values correspond to average PSNR in dB on the testing set.}\label{tab:all_results}

\end{table}

\paragraph{Inpainting with CelebA}
We evaluate the unsupervised methods in~\Cref{sec: algos} on the CelebA dataset~\cite{liu2015faceattributes}, which is split into 32556 images for training and 32556 images for testing.  We use the same U-Net (see~\Cref{sec: training details}) for  MOI and supervised learning, and the DCGAN architecture~\cite{radford2015unsupervised} for AmbientGAN as in~\cite{bora2018ambientgan}. \rev{We observed that using the U-Net for AmbientGAN's generator achieves worse results than the DCGAN architecture.} Reconstructed test images are shown in~\Cref{fig:inpainting_images} and average test PSNR is presented in~\cref{tab:all_results}.
The proposed method obtains an improvement of more than 4 dB with respect to  AmbientGAN and falls only $2.1$ dB behind supervised learning.

\paragraph{Accelerated MRI with FastMRI}
Finally, we consider the FastMRI dataset~\cite{knoll2020fastMRI}, where the set of forward operators $A_g$ consist of different sets of single-coil $k$-space measurements, with $4\times$ acceleration, i.e., $m/n= 0.25$. We used $900$ images for training and 74 for testing, which we split across $G=40$ operators. We compare measurement splitting, MOI and supervised learning, all using the same \JT{denoiser or unrolled architecture (see Appendix B for details). For measurement splitting,} we follow the strategy in~\cite{yaman2020mri}, and choose to assign a random subset representing $60\%$ of the measurements in $A_g$ to $A_{g,1}$ and  the remaining to $A_{g,2}$. We observed that a model trained with measurement splitting obtained less test error using reduced measurements associated with $A_{g,1}$ instead of the full measurements $A_{g}$, so we report the best\footnote{\JT{Using the full measurements with the denoiser architecture obtains an average test PSNR of 27.3 dB, i.e., a decrease of 1.4 dB  with respect to the performance using $A_{g,1}$ presented in~\Cref{tab:all_results}.}} results using $A_{g,1}$. As observed in~\cite{antun2020instabilities}, the network  fails to generalize to operators with more measurements. 
 Average test PSNR is presented in~\Cref{tab:all_results}. Reconstructed test images with the unrolled architecture are shown in~\Cref{fig:mri_images_unrolled}. 
The proposed method \JT{performs better than measurement splitting while obtaining results close to the} supervised setting. \JT{All training approaches obtain a better performance \JT{with the unrolled architecture} than using the denoising network due to the  architectural improvements.}

 \begin{figure}[h]
\centering\centering
\begin{minipage}{.195\linewidth}
\centerline{\includegraphics[width=1\textwidth]{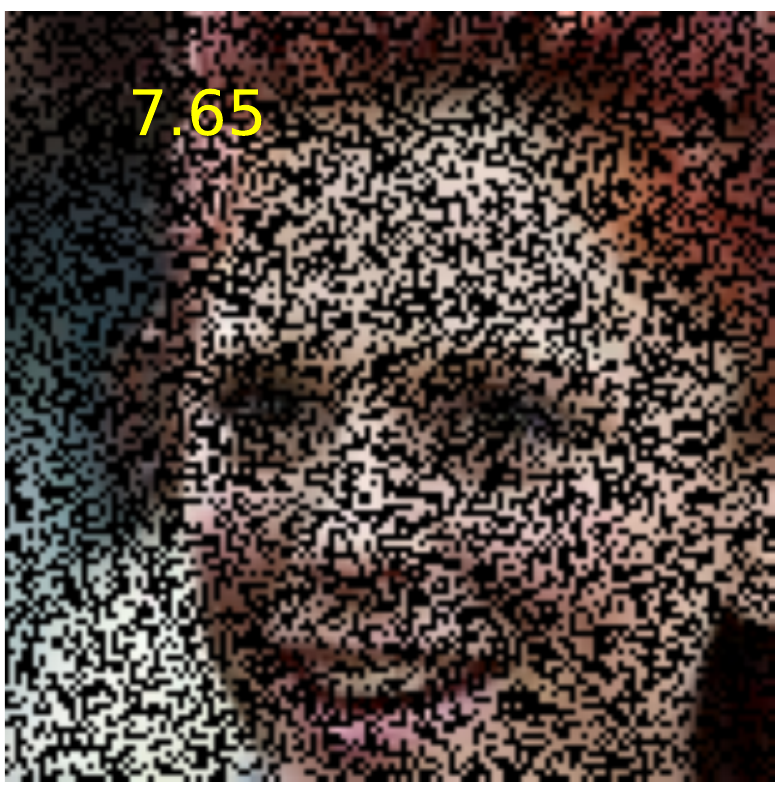}}
\end{minipage}
\begin{minipage}{.195\linewidth}
\centerline{\includegraphics[width=1\textwidth]{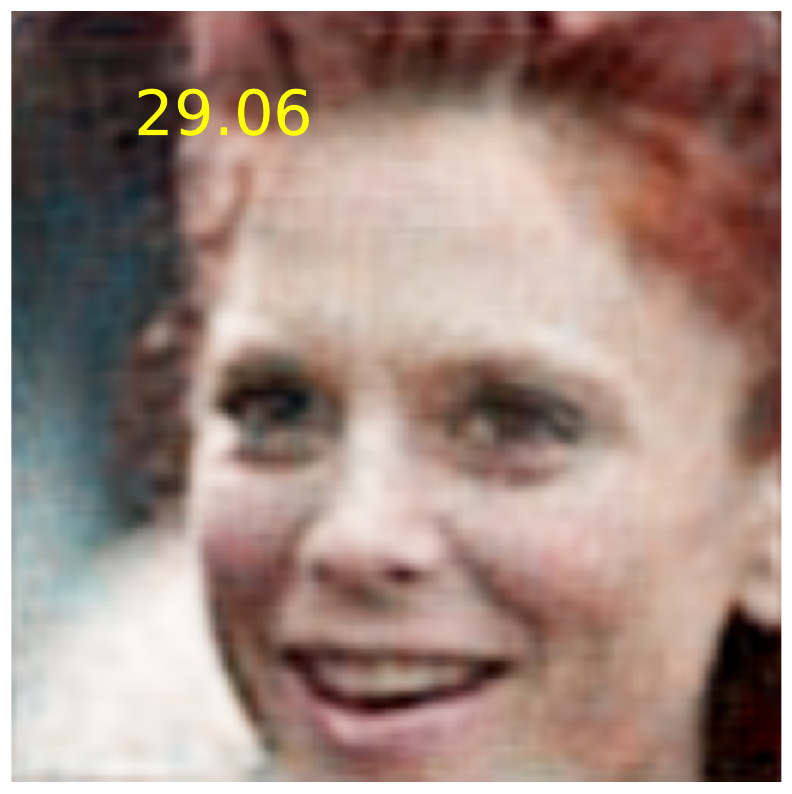}}
\end{minipage}
\begin{minipage}{.195\linewidth}
\centerline{\includegraphics[width=1\textwidth]{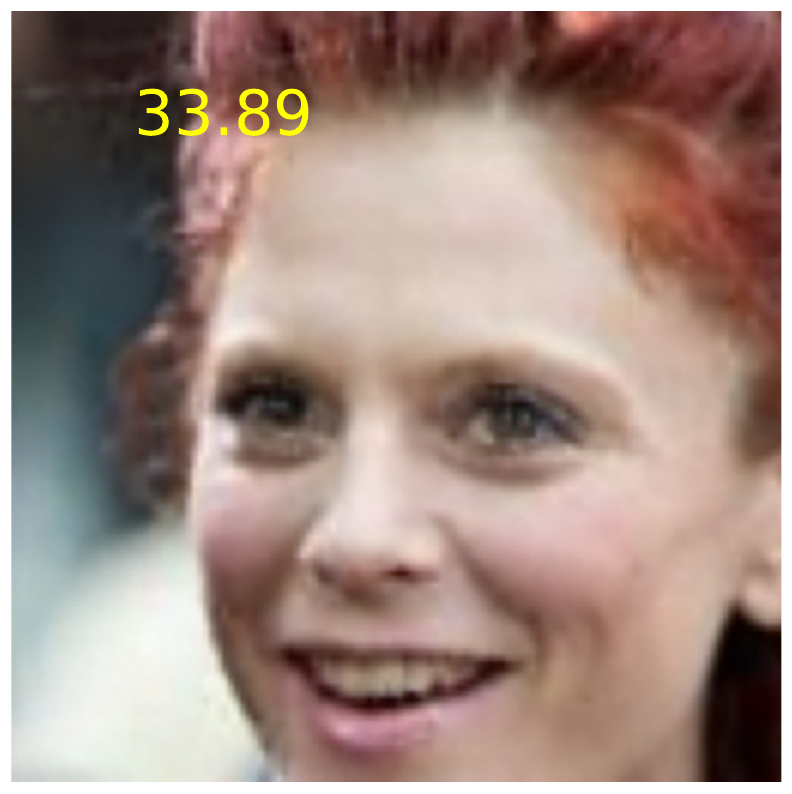}}
\end{minipage}
\begin{minipage}{.195\linewidth}
\centerline{\includegraphics[width=1\textwidth]{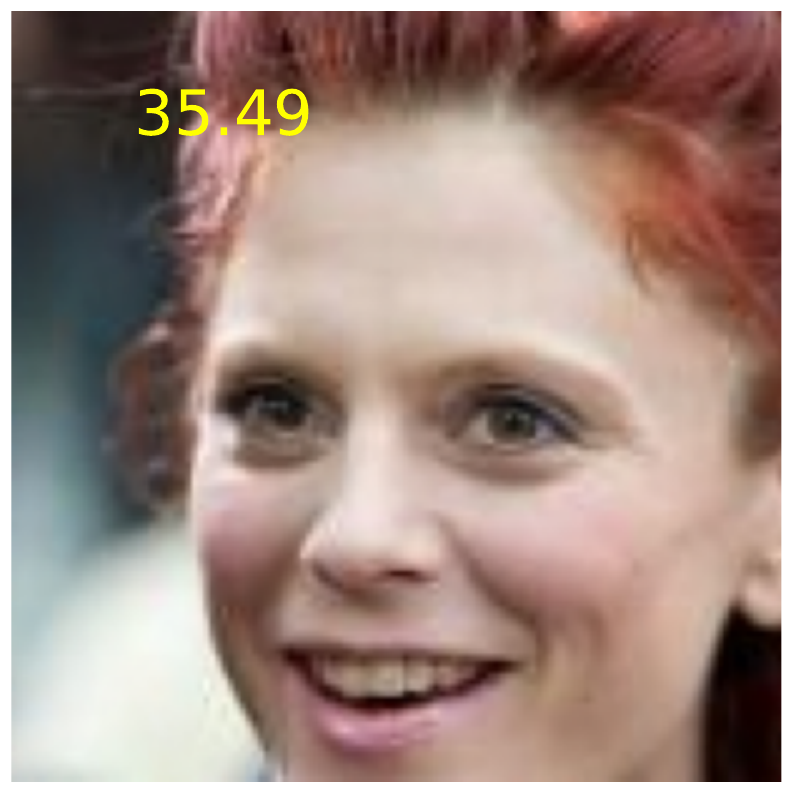}}
\end{minipage}
\begin{minipage}{.195\linewidth}
\centerline{\includegraphics[width=1\textwidth]{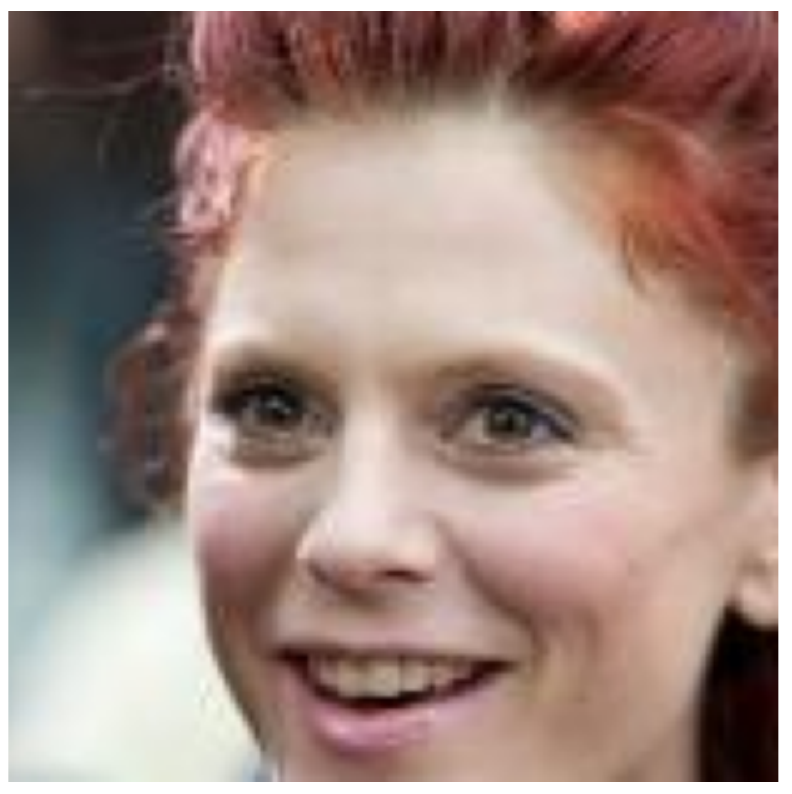}}
\end{minipage}
\caption{Reconstructed test images for the inpainting task using the CelebA dataset. From left to right: pseudo-inverse $A^{\dagger}y$, AmbientGAN, MOI, supervised and ground-truth. PSNR values are reported in yellow. }
\label{fig:inpainting_images}
\end{figure}

 \begin{figure}[h]
\centering\centering
\begin{minipage}{.195\linewidth}
\centerline{\includegraphics[width=1\textwidth]{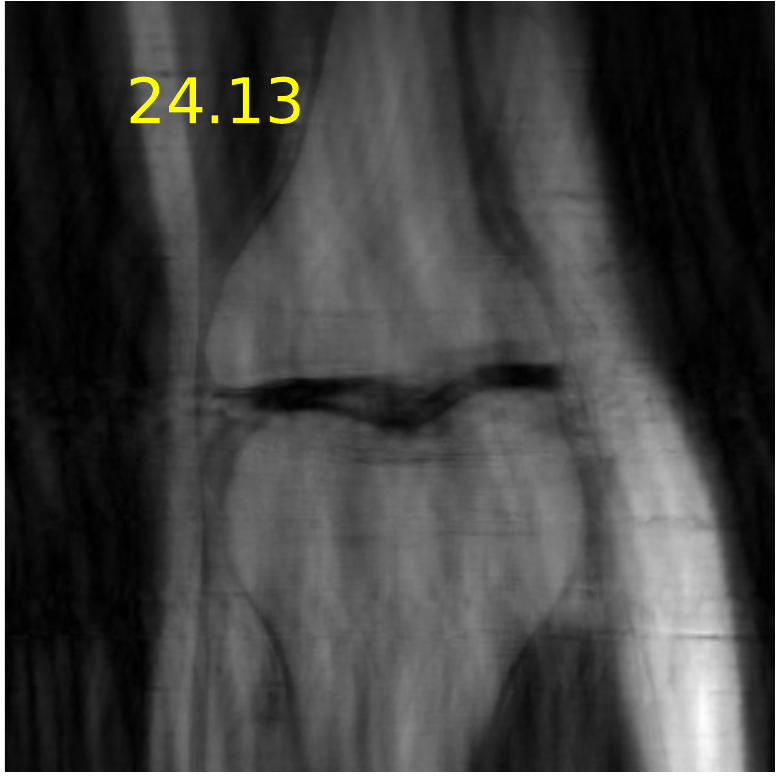}}
\end{minipage}
\begin{minipage}{.195\linewidth}
\centerline{\includegraphics[width=1\textwidth]{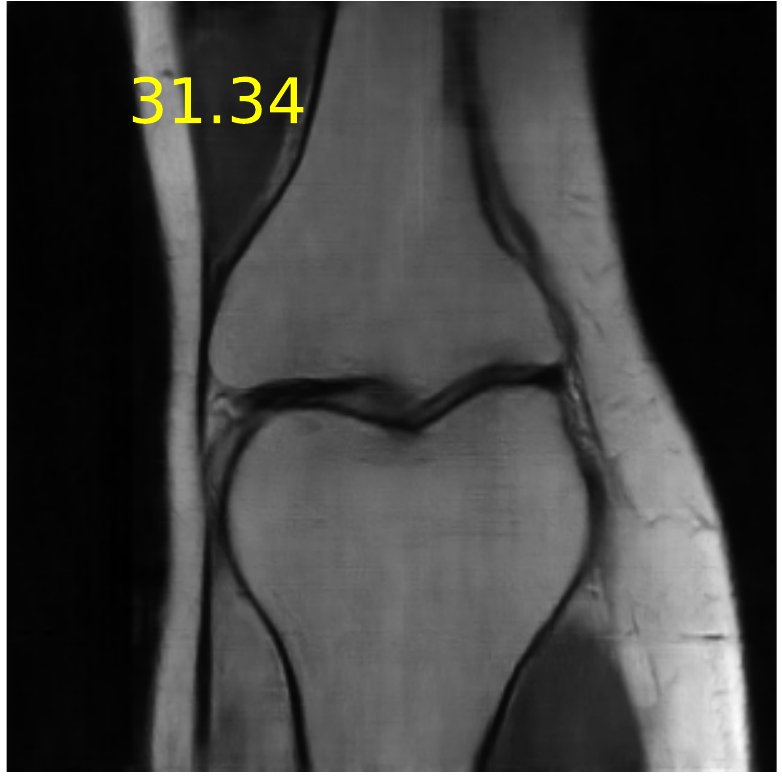}}
\end{minipage}
\begin{minipage}{.195\linewidth}
\centerline{\includegraphics[width=1\textwidth]{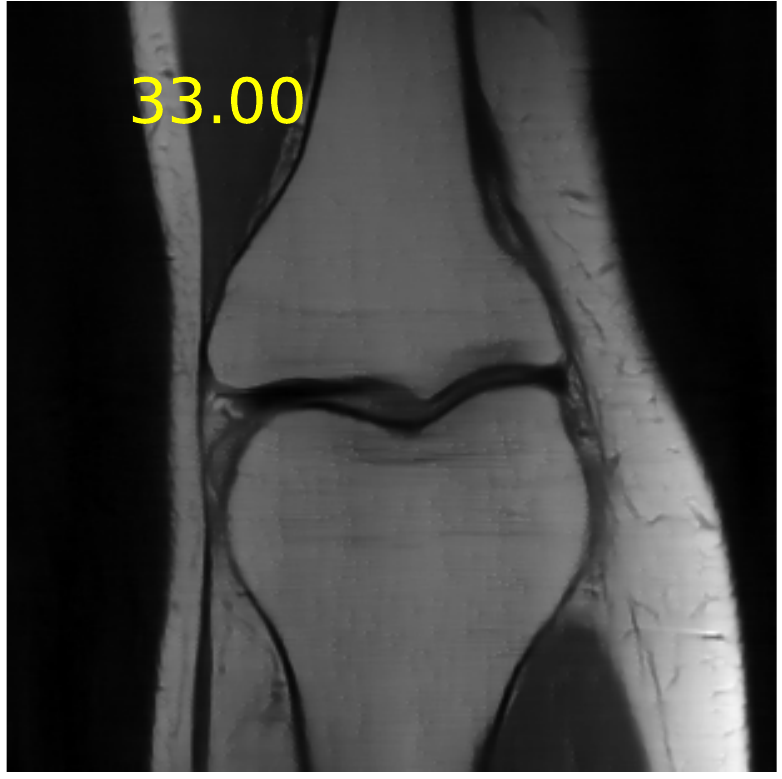}}
\end{minipage}
\begin{minipage}{.195\linewidth}
\centerline{\includegraphics[width=1\textwidth]{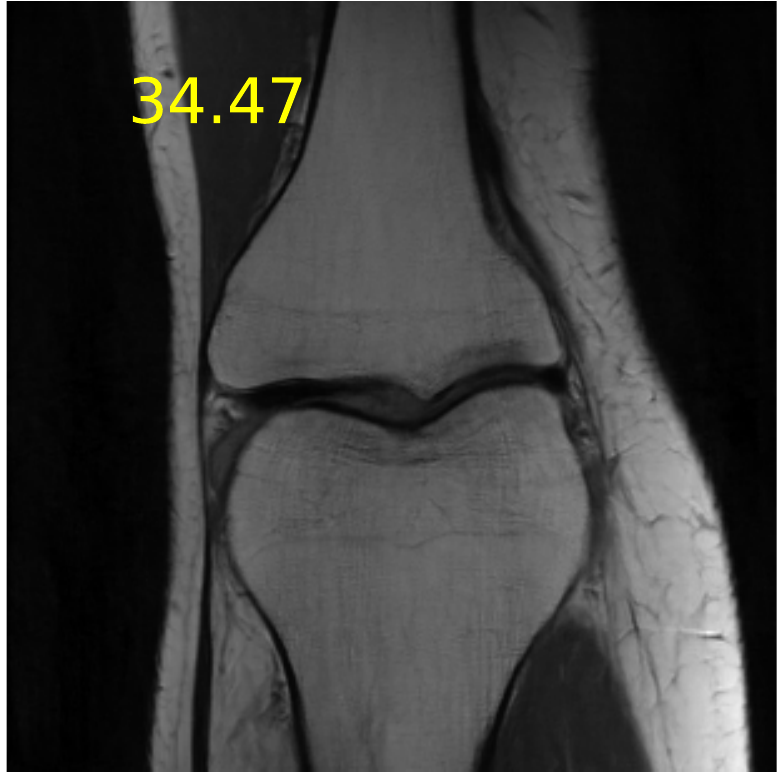}}
\end{minipage}
\begin{minipage}{.195\linewidth}
\centerline{\includegraphics[width=1\textwidth]{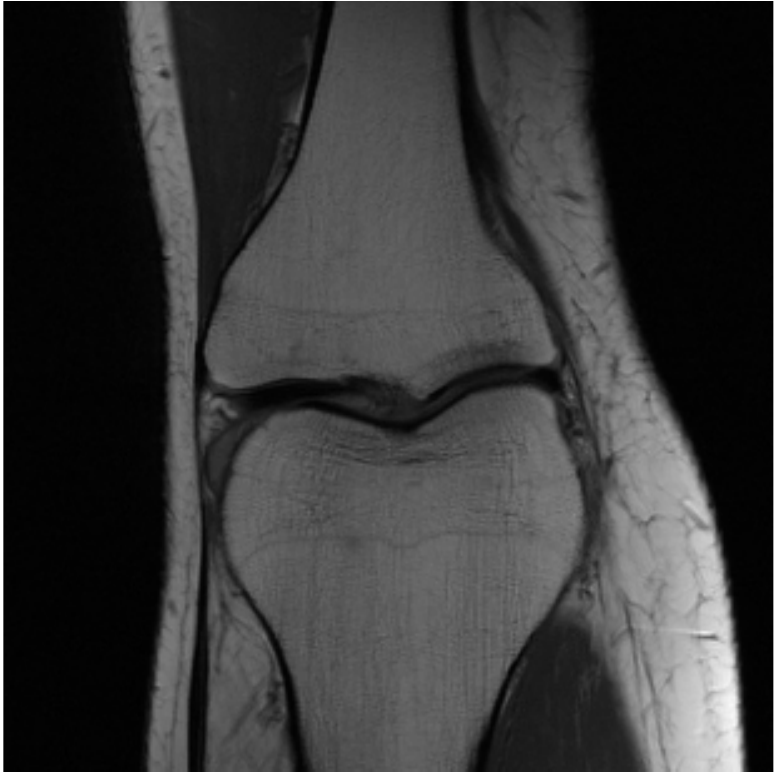}}
\end{minipage}
\caption{Examples of reconstructed test images for the accelerated MRI task using an unrolled network architecture (PGD-3). From left to right: pseudo-inverse $A^{\dagger}y$, measurement splitting, MOI, supervised and ground-truth. PSNR values are reported in yellow. }
\label{fig:mri_images_unrolled}
\end{figure}

\section{Limitations}
\Cref{theo: multiple op} does not cover cases where the operators $A_g$ present some problem specific constraints (e.g., they are inpainting matrices) as well as cases where the signal model is only approximately low dimensional. 
Note however that~\cref{prop: necessary multA} applies to constrained operators. We leave the study of sufficient conditions for these particular cases for future work. 
The proposed loss might not be effective in problems where learning the reconstruction function is impossible, e.g., due to very high noise affecting the measurements~\cite{gupta2020multi}. In this particular case,  it might be possible to learn a generative model as in AmbientGAN~\cite{bora2018ambientgan}.

\section{Conclusions} 
We presented sensing theorems for the unsupervised learning of signal models from incomplete measurements using multiple measurement operators. Our bounds characterize the interplay between the  fundamental properties of the problem: the ambient dimension, the data dimension and the number of measurement operators. The bounds are agnostic of the learning algorithms and provide useful necessary and sufficient conditions for designing principled sensing strategies.

Furthermore, we presented a new practical unsupervised learning loss which learns to reconstruct incomplete measurement data from multiple operators, outperforming previously proposed unsupervised methods. The proposed strategy avoids the adversarial training in AmbientGAN, which can suffer from mode collapse~\cite{arora2018do}, and, contrary to measurement splitting, is trained using full operators $A_g$.
Our results shed light into the setting where access to ground truth data cannot be guaranteed which is of extreme importance in various applications.

\section*{Acknowledgements}
\JT{This work is supported by the ERC C-SENSE project (ERCADG-2015-694888).}

\bibliography{biblio}
\bibliographystyle{unsrt}

\section*{Checklist}


\begin{enumerate}

\item For all authors...
\begin{enumerate}
  \item Do the main claims made in the abstract and introduction accurately reflect the paper's contributions and scope?
    \answerYes{}
  \item Did you describe the limitations of your work?
    \answerYes{}
  \item Did you discuss any potential negative societal impacts of your work?
    \answerYes{}
  \item Have you read the ethics review guidelines and ensured that your paper conforms to them?
    \answerYes{}
\end{enumerate}

\item If you are including theoretical results...
\begin{enumerate}
  \item Did you state the full set of assumptions of all theoretical results?
    \answerYes{}
        \item Did you include complete proofs of all theoretical results?
    \answerYes{}
\end{enumerate}

\item If you ran experiments...
\begin{enumerate}
  \item Did you include the code, data, and instructions needed to reproduce the main experimental results (either in the supplemental material or as a URL)?
    \answerYes{}
  \item Did you specify all the training details (e.g., data splits, hyperparameters, how they were chosen)?
    \answerYes{}
        \item Did you report error bars (e.g., with respect to the random seed after running experiments multiple times)?
    \answerYes{}
        \item Did you include the total amount of compute and the type of resources used (e.g., type of GPUs, internal cluster, or cloud provider)?
    \answerYes{}
\end{enumerate}

\item If you are using existing assets (e.g., code, data, models) or curating/releasing new assets...
\begin{enumerate}
  \item If your work uses existing assets, did you cite the creators?
    \answerYes{}
  \item Did you mention the license of the assets?
    \answerYes{}
  \item Did you include any new assets either in the supplemental material or as a URL?
    \answerNA{} We used existing models and datasets.
  \item Did you discuss whether and how consent was obtained from people whose data you're using/curating?
    \answerYes{} All datasets considered in the paper are publicly available.
  \item Did you discuss whether the data you are using/curating contains personally identifiable information or offensive content?
    \answerYes{}
\end{enumerate}

\item If you used crowdsourcing or conducted research with human subjects...
\begin{enumerate}
  \item Did you include the full text of instructions given to participants and screenshots, if applicable?
    \answerNA{} We did not use any crowdsourcing or conducted research with human subjects.
  \item Did you describe any potential participant risks, with links to Institutional Review Board (IRB) approvals, if applicable?
    \answerNA{}
  \item Did you include the estimated hourly wage paid to participants and the total amount spent on participant compensation?
    \answerNA{}
\end{enumerate}

\end{enumerate}

\newpage
\appendix

\section{Proofs}\label{sec:proofs} 

The proof of~\Cref{theo: multiple op} utilises the following technical lemma:
\begin{lemma} [Lemmas 4.5 and 4.6 in \cite{sauer1991embedology}] \label{lemma:sauer}

Let $S$ be a bounded subset of $\R{n}$, and let $G_0, G_1,\dots,G_t$ be Lipschitz maps from $S$ to $\R{m}$. For each integer $r\geq0$, let $S_r$ be the subset of $z\in S$ such that the rank of the $m\times t$ matrix 
\begin{equation}
    \Phi_z = [G_1(z),\dots,G_t(z)]
\end{equation}
is $r$, and let $\text{boxdim}(S_r) = k_r$. For each $\alpha\in\R{t}$ define $G_{\alpha}(z)=G_0 + \Phi_z\alpha $. 
If for all integers $r\geq 0$ we have that $r > k_r$, then $G_{\alpha}^{-1}(0)$ is empty for almost every $\alpha\in \R{t}$.
\end{lemma}

\begin{proof}
The proof of \cref{lemma:sauer} follows standard covering arguments and may be sketched as follows. From the dimensionality assumption, the set $S_r$ can be essentially covered by  $\mathcal{O}(\epsilon^{-k_r})$ $\epsilon$-balls. Furthermore, for any $z \in S_r$, the probability (measured with respect to $\alpha \in \R{t}$) that $G_{\alpha}(z)$ maps to the neighborhood of $0$ scales as $\epsilon^r$. Hence the probability of this happening for any of the points in the cover scales as $\epsilon^{r-k_r}$. If we take $r> k_r$ then the probability of such an event tends to zero as we shrink $\epsilon$. Full details can be found in the proofs in~\cite{sauer1991embedology}.
\end{proof}

We can now present the proof of~\Cref{theo: multiple op}:
\begin{proof}
In order to have model uniqueness, we require that the inferred signal set $\hat{\signalset}$ defined as
\begin{equation}
    \label{eq:inferred set2}
  \hat{\signalset} =  \{ v\in \mathbb{R}^{n} | \; A_g(x_{g} -v) = 0, \; g=1,\dots,\ntransf, x_1,\dots,x_{\ntransf}\in \signalset \}
\end{equation} 
equals the true set $\signalset$, or equivalently that their difference 
\begin{multline} \label{eq:diff inf}
    \hat{\signalset}\setminus{\signalset} = \{ v\in \R{n} \setminus{\signalset} | \;  A_1 (x_1 - v) = \dots =  
    A_{\ntransf} (x_{\ntransf} -v) = 0, \; x_1,\dots,x_{\ntransf}\in \signalset \}
\end{multline}
is empty, where $\setminus{}$ denotes set difference. Let $S\subset \R{n(\ntransf+1)}$ be the set of all vectors $z=[v,x_1,\dots,x_{\ntransf}]^{\top}$ with  $v\in \R{n}\setminus{\signalset}$ and $x_1,\dots,x_{\ntransf}\in \signalset$. The difference set defined in \Cref{eq:diff inf} is empty if and only if for any $z\in S$  we have
\begin{align}  \label{eq:multA}
  \underbrace{\begin{bmatrix}
  - A_1 & A_{1} &  & \\ 
   \vdots & & \ddots & \\ 
  - A_{\ntransf}& &  & A_{\ntransf}
 \end{bmatrix}}_{G_\alpha\in\R{m\ntransf \times n(\ntransf+1) }}
  \underbrace{\begin{bmatrix}
  v \\
   x_{1}\\
   \vdots \\ 
    x_{\ntransf}\\
 \end{bmatrix}}_{z\in S } &\neq 0 \\
 G_{\alpha}(z) &\neq 0
\end{align}
where $G_\alpha$ maps $z\in S$ to  $\R{m\ntransf}$. Let $\alpha = [\vect{A_1}^{\top},\dots,\vect{A_{\ntransf}}^{\top}]^{\top}\in \R{mn\ntransf}$, then as a function of $\alpha$ we can also write \Cref{eq:multA} as 
 \begin{equation} \label{eq: alpha mult}
  \begin{bmatrix}
   (x_1-v)^{\top} \otimes I_{m} &  & \\ 
   & \ddots & \\ 
   &  & (x_{\ntransf}-v)^{\top} \otimes I_{m}
 \end{bmatrix} \alpha \neq 0
\end{equation}
where $\otimes$ is the Kronecker product and we used the fact that $A(x_g-v)= (x_g-v)^{\top} \otimes I_{m} \vect{A}$. As $v$ does not belong to the signal set, the matrix on the left hand side of \Cref{eq: alpha mult} has rank $m\ntransf$ for all  $z\in S$. We treat the cases of bounded and conic signal sets separately, showing in both cases that, for almost every $\alpha\in \R{mn\ntransf}$, the condition in \Cref{eq: alpha mult} holds for all $z\in S$ if $m>k+n/\ntransf$:


\begin{description}
    \item[Bounded signal set] Let $S_\rho$ be a subset of $S$ defined as
    \begin{multline}
      S_\rho =  \{ z \in \R{n(\ntransf+1)} \; | z = [v^{\top},x_1^{\top},\dots,x_{\ntransf}^{\top} ]^{\top}, x_1,\dots,x_{\ntransf}\in \signalset,\; \|v\|_2\leq \rho \}.
    \end{multline}
    As $S_\rho$ is bounded, we have $\text{boxdim}(S_\rho)\leq k\ntransf+n$. Thus, if $m\ntransf>k\ntransf+n$,  \Cref{lemma:sauer} states that for almost every $\alpha$, \Cref{eq: alpha mult} holds for all $z\in S_\rho$. As $S$ can be decomposed as a countable union of $S_\rho$ of increasing radius, i.e., $S=\bigcup_{\rho\in\mathbb{N}} S_{\rho}$, and a countable union of events of measure zero has measure  zero, then for almost every $\alpha$ all $z\in S$ verifies \Cref{eq: alpha mult} if $m>k+n/\ntransf$.
    \item[Conic signal set] If the signal set is conic, then $S$ is also conic. Hence, due to the linearity of \Cref{eq:multA} with respect to $z$, there exists $z\in S$ which does not verify \Cref{eq:multA} if and only if for any bounded set $B$ containing an open neighbourhood of $0$, there exists a $z\in S\cap B$ which does not verify \Cref{eq:multA}. As $\text{boxdim}(S\cap B)\leq \ntransf k+n$, \Cref{lemma:sauer} states that for almost every $\alpha$, all $z\in S$ verifies \Cref{eq: alpha mult} as long as $m> k + n\ntransf$. 
\end{description}
\end{proof}

We end this section with the proof of~\Cref{prop: necessary multA}:
\begin{proof}
Consider the noisy measurements associated to the $g$th operator $A_g$, as $z=y+\epsilon$, where $z$ are the observed noisy measurements, $y$ are the clean measurements and $\epsilon$ is additive noise (independent of $y$). 
The characteristic function of the sum of two independent random variables is given by the multiplication of their characteristic functions, i.e.,
\begin{equation}
    \varphi_z(w) = \varphi_y(w) \varphi_\epsilon(w)
\end{equation}
where $\varphi_z$, $\varphi_y$ and $\varphi_\epsilon$ are the characteristic functions the noisy measurement, clean measurements and noise distributions, respectively. If the characteristic function of the noise distribution is nowhere zero, we can uniquely identify the characteristic function of the clean measurement distribution as 
\begin{equation}
 \varphi_y(w)   = \varphi_z(w)/ \varphi_\epsilon(w)
\end{equation}
The clean measurement distribution is fully characterized by its characteristic function $\varphi_y(w)$. We end the proof by noting that the same reasoning applies to the measurements of every operator $A_g$ with $g\in \{1,\dots,\ntransf\}$.
\end{proof}

\section{Training Details} \label{sec: training details}

Algorithm~\ref{algo:moi} provides the pseudo-code of the proposed multi-operator imaging (MOI) method. The training details for each task are as follows:

\begin{algorithm}[h]
\scriptsize
\SetAlgoLined
\ttfamily\textcolor{teal}{}\\
\ttfamily\textcolor{teal}{\#\hspace{1.5mm}$\mathcal{G}$:\hspace{1.5mm}forward operators $\mathcal{G}$=\{$1,\dots,G$\}}\\
\ttfamily\textcolor{teal}{\#\hspace{1.5mm}f:\hspace{1.5mm}reconstruction function (e.g., neural network)}\\
\ttfamily{for y, $\text{A}_\text{g}$ in loader:}\ttfamily\textcolor{teal}{\hspace{1.5mm}\#\hspace{1.5mm}load a minibatch y with N samples and its corresponding operator $\text{A}_\text{g}$}\\
\hspace*{4.3mm}\ttfamily\textcolor{teal}{\#\hspace{1.5mm}randomly select a operator from $\mathcal{G}/g$}\\
\hspace*{4.3mm}\ttfamily{s = select($\mathcal{G}/g$)}\\
\hspace*{4.3mm}\ttfamily{x1 = f(y, $\text{A}_\text{g}$)}\ttfamily\textcolor{teal}{\hspace{1.5mm}\#\hspace{1.5mm}reconstruct x from y}\\
\hspace*{4.3mm}\ttfamily{x2 = f($\text{A}_\text{s}$(x1), $\text{A}_\text{s}$)}\ttfamily\textcolor{teal}{\hspace{1.5mm}\#\hspace{1.5mm}reconstruct x1}\\
\hspace*{4.3mm}\ttfamily\textcolor{teal}{\#\hspace{1.5mm}MOI training loss, Eqn.(18)}\\
\hspace*{4.3mm}\ttfamily{loss = MSELoss($\text{A}_\text{g}$(x1), y)}\ttfamily\textcolor{teal}{\hspace{1.5mm}\#\hspace{1.5mm}measurement consistency}\\
\hspace*{11.75mm}\ttfamily{+ MSELoss(x2, x1)}\ttfamily\textcolor{teal}{\hspace{1.5mm}\#\hspace{1.5mm}cross-operator consistency}\\
\hspace*{4.3mm}\ttfamily\textcolor{teal}{\#\hspace{1.5mm}update f network}\\
\hspace*{4.3mm}\ttfamily{loss.backward()}\\
\hspace*{4.3mm}\ttfamily{update(f.params)}\\
\caption{Pseudocode of MOI in a PyTorch-like style.}\label{algo:moi}
\end{algorithm}

\paragraph{Compressed Sensing and Inpainting with MNIST.}
In both cases, we use the Adam optimizer with a batch size of $128$ and weight decay of $10^{-8}$. We use a fully connected network with 5 layers, where the number neurons in each layer are $784,1000,32,1000,784$ respectively. The nonlinearity is relu and the network has a residual connection between the input and output. 
For the CS task, we use an initial learning rate of $10^{-4}$ and train the networks for $1000$ epochs,  keeping the learning rate constant for the first 800 epochs and then shrinking it by a factor of $0.1$.  For the inpainting task, we use an initial learning rate of $5\times10^{-4}$ and train the networks for $500$ epochs, keeping the learning rate constant for the first $300$ epochs and then shrinking it by a factor of $0.1$.

\paragraph{Inpainting with CelebA.} The CelebA  dataset  contains more than 200K celebrity images, each with 40 binary attributes. We pick the attribute \emph{smile} to evaluate the proposed method. The center part of the aligned images in the CelebA dataset are cropped to $128\times128$. We divide the selected images into two subsets for training and testing. There are $32557$ images in each subset, which we split across $G=40$ operators. For the inpainting task, we use the U-Net architecture (see~\cref{fig:U-Net}) to implement the reconstruction function $f(y,A)$, and the DCGAN architecture for AmbientGAN. \rev{Using the U-Net architecture for AmbientGAN's generator obtains an average test PSNR of $27.5\pm 1.3$ dB, which is 2.1 dB below the performance obtained by the DCGAN generator reported in Section 6.} We use Adam with a batch size of $20$, an initial learning rate of $5\times10^{-4}$ and a weight decay of $10^{-8}$. We train the networks for $300$ epochs,  shrinking the learning rate by a factor of $0.1$ after the first $200$ epochs.

\begin{figure}
\centering
\includegraphics[width=0.45\textwidth]{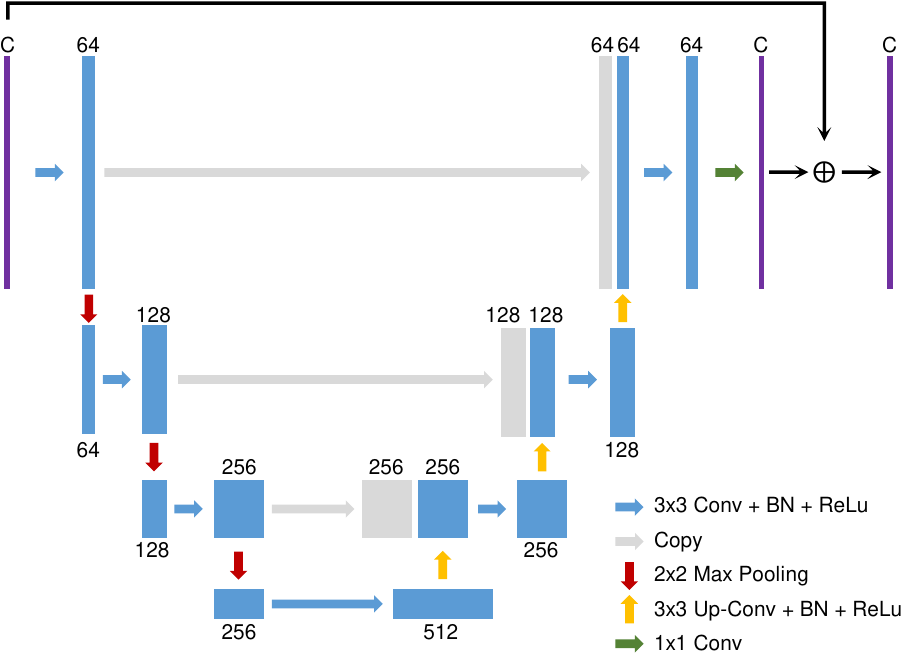}
\caption{The residual U-Net used in the paper. The number of input and output channels is denoted as $C$, such that $C=2$ for MRI and  $C=3$ for inpainting.}
\label{fig:U-Net}
\end{figure}

\paragraph{Accelerated MRI with fastMRI.} 

\JT{For the denoiser network $f(y,A)=\tilde{f}(A^{\dagger}y)$, we use the U-Net in~\cref{fig:U-Net} to implement $\tilde{f}$.
For the unrolled network, we unfold the proximal gradient descent (PGD) algorithm (see Algorithm~\cref{alg:pgd}) with $T=3$ iterations.} The step size is initialized as $\eta^{(t)}=0.4$ and is then learned during training. We employ 3 U-Net networks using the architecture in~\cref{fig:U-Net} to implement $f^{(t)}$ for $t=1,2,3$ (no weight sharing across PGD iterations).

\begin{equation}\label{alg:pgd}
\left\lfloor
  \begin{array}{llr}
  \text{Unrolled Proximal Gradient Descent (PGD)}& & \\
  \textbf{input: } y, A & & \\
  x^{(0)}\leftarrow A^{\dagger}y& & \\
  \text{for}~t=0,1,\cdots,T-1:&&\\
    x^{(t+1)}\leftarrow f^{(t+1)} (x^{(t)} - \eta^{(t)} A^{\top} (Ax^{(t)}-y))&&\\
  \text{end } \text{for}& &\\
  \text{return } f(y, A):=x^{(T)}& &\\
  \end{array}
\right.
\end{equation}

We train the networks for $500$ epochs \JT{with the Adam optimizer (batch size of 4), with initial learning rate $5\times 10^{-4}$ for} the first 300 epochs and then shrinking it by a factor of $0.1$. In all experiments, we use complex-valued data and treat real and imaginary parts of the images as separate channels.  For the purpose of visualization, we display only the magnitude images. \JT{Reconstructed test images for the denoiser and unrolled architectures are shown in~\cref{fig:mri_images,fig:mri_images_unrolled} respectively.}

 
 \begin{figure}[h]
\centering
\begin{minipage}{.195\linewidth}
\centerline{\includegraphics[width=.98\textwidth]{figures/nips/sm/unet/fig_a2a_mri_fbp.pdf}}
\end{minipage}
\begin{minipage}{.195\linewidth}
\centerline{\includegraphics[width=.98\textwidth]{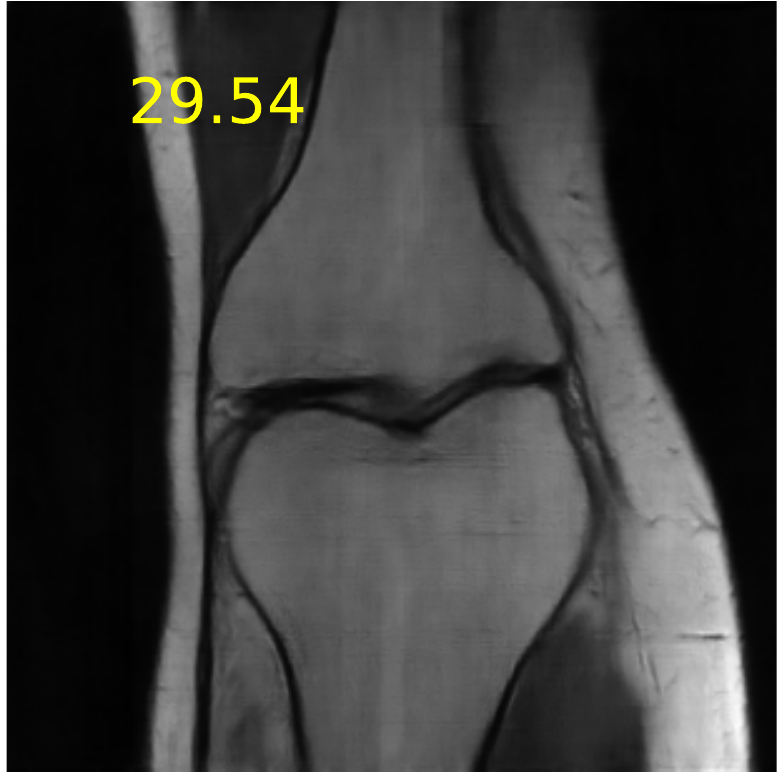}}
\end{minipage}
\begin{minipage}{.195\linewidth}
\centerline{\includegraphics[width=.98\textwidth]{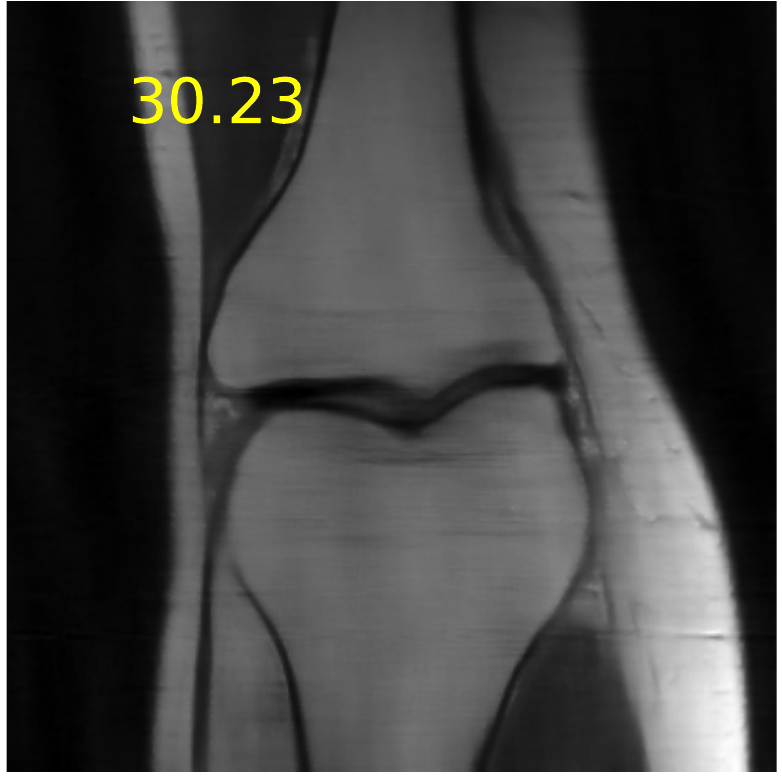}}
\end{minipage}
\begin{minipage}{.195\linewidth}
\centerline{\includegraphics[width=.98\textwidth]{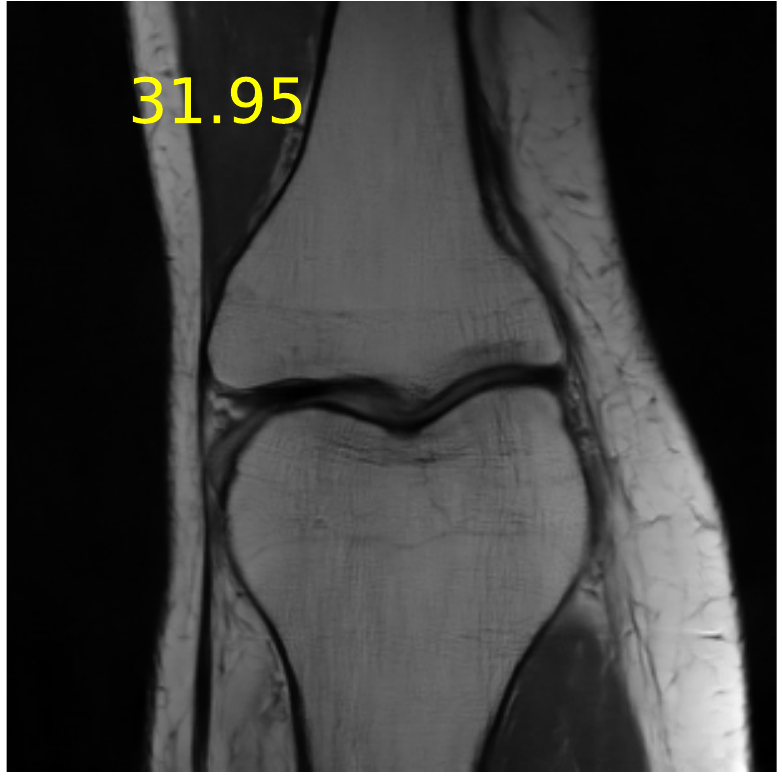}}
\end{minipage}
\begin{minipage}{.195\linewidth}
\centerline{\includegraphics[width=.98\textwidth]{figures/nips/sm/unet/fig_a2a_mri_gt.pdf}}
\end{minipage}

\caption{Examples of reconstructed test images for the accelerated MRI task using the fastMRI dataset using the denoiser architecture. From left to right: pseudo-inverse $A^{\dagger}y$, measurement splitting, MOI, supervised and ground-truth. 
}
\label{fig:mri_images}
\end{figure}

 \begin{figure*}[h]
\centering
\includegraphics[width=1\textwidth]{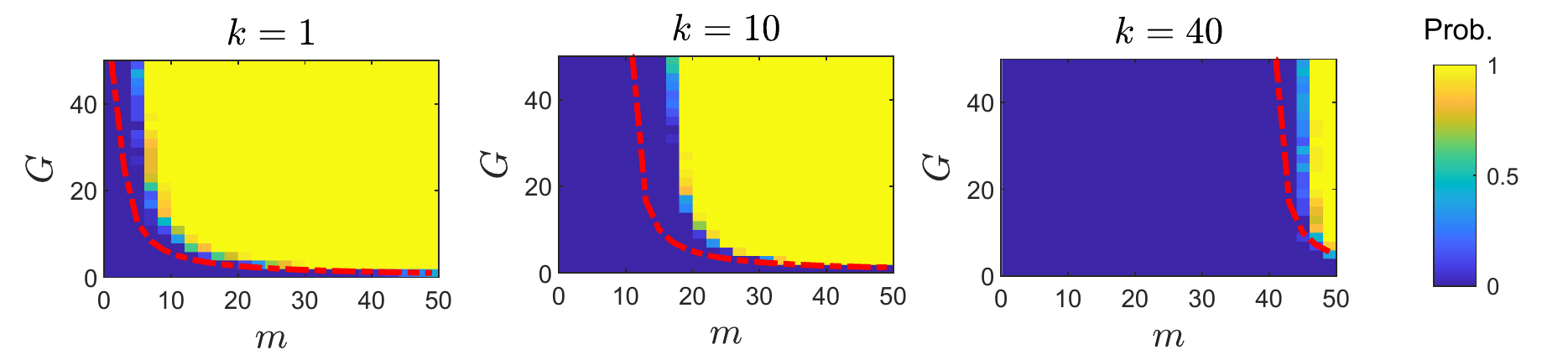}
\caption{Reconstruction probability of a $k$-dimensional subspace using incomplete measurements arising from $\ntransf$ independent operators for different $k$. The curve in red shows the bound of~\Cref{theo: multiple op}, $m>k+n/\ntransf$.}
\label{fig:multG_results}
\end{figure*}

\section{Additional Experiments} \label{sec: additional exp}
\subsection{Subspace Learning}\label{subsec: subspace}
We consider the problem of learning a $k$-dimensional subspace model from partial observations, where the signals $x_i$ are generated from a standard Gaussian distribution on the low-dimensional subspace. 
The observations $y_i$ are obtained by randomly choosing one out of $\ntransf$ operators $A_1,\dots,A_{\ntransf}\in\R{m \times n}$, each composed of iid Gaussian entries of mean 0 and variance $n^{-1}$. In order to recover the signal matrix $X = [x_1,\dots,x_N]$, we solve the following low-rank matrix recovery problem 
\begin{align}\label{eq:matrix comp}
    \arg\min_{X} \;&\|X\|_{*} \\
    \text{s.t. } A_{g_i}x_i &= y_i \quad \forall i=1,\dots,N \nonumber
\end{align}
where $\|\cdot\|_{*}$ denotes the nuclear norm. A recovery is considered successful if $\frac{\sum_i\|\hat{x}_i-x_i\|^{2}}{\sum_i\|x_i\|^{2}}<10^{-1}$, where $\hat{x}_i$ is the estimated signal for the $i$th sample. We use a standard matrix completion algorithm~\cite{cai2010singular} to solve~\cref{eq:matrix comp}. The ambient dimension is fixed at $n=50$, and the experiment is repeated for $k=1,10,40$. For each experiment we set $N=150k$ in order to have enough samples to estimate the subspaces~\cite{pimentel2016characterization}. \Cref{fig:multG_results} shows the probability of recovery over $25$ Monte Carlo trials for different numbers of measurements $m$ and operators $\ntransf$. The reconstruction probability exhibits a sharp transition which follows the bound presented in~\Cref{theo: multiple op}, i.e.,  $m>k+n/\ntransf$. 


\end{document}